%% file: main.tex
\documentclass{article}

\usepackage[accepted]{icml/icml2022}
\input{macros}

\input{packages}

\begin{document}
\twocolumn[\vspace{-15pt}
\icmltitle{Kernel Interpolation as a Bayes Point Machine}

\hspace{0.05\textwidth}
\begin{minipage}{0.3\textwidth}\begin{center}
    \textbf{Jeremy Bernstein}\\
    {\small\texttt{bernstein@caltech.edu}}
\end{center}\end{minipage}
\begin{minipage}{0.3\textwidth}\begin{center}
    \textbf{Alex Farhang}\\
    {\small\texttt{afarhang@caltech.edu}}
\end{center}\end{minipage}
\begin{minipage}{0.3\textwidth}\begin{center}
    \textbf{Yisong Yue}\\
    {\small\texttt{yyue@caltech.edu}}
\end{center}\end{minipage}

\icmlkeywords{neural networks, kernels, Gaussian processes, PAC-Bayes, convex geometry, Bayes point machines, Gibbs error, Bayes error, mean voter theorem}

\vskip 0.3in
]

\input{section/00-abstract}
\input{section/01-intro}
\input{section/02-related}
\input{section/03-bpmtheory}
\input{section/04-nn-gp-k}
\input{section/05-pac-bayes}
\input{figure/gp}
\input{section/06-experiments}
\input{section/08-discuss}

\bibliography{refs}
\bibliographystyle{icml/icml2022}

\end{document}

%% file: macros.tex
\newcommand\com{%
    \tikz[radius=0.2em] {%
        \fill (0,0) -- ++(0.2em,0) arc [start angle=0,end angle=90] -- ++(0,-0.4em) arc [start angle=270, end angle=180];%
        \draw (0,0) circle;%
    }%
}

\usepackage{dsfont}
\usepackage{stmaryrd}

\usepackage{verbatim}

\newcommand{\gibbserr}{\eps_{\mathrm{Gibbs}}}
\newcommand{\bayeserr}{\eps_{\mathrm{Bayes}}}
\newcommand{\bpmerr}{\eps_{\mathrm{BPM}}}
\newcommand{\kl}{\mathrm{KL}}

\newcommand{\econst}{\mathrm{e}}

\newcommand{\eps}{\varepsilon}

\renewcommand{\phi}{\varphi}

\newcommand{\half}{\tfrac{1}{2}}

\newcommand{\Q}{\mathbb{Q}}
\newcommand{\R}{\mathbb{R}}

\newcommand{\sign}{\operatorname{sign}}

\newcommand{\trace}{\operatorname{tr}}

\newcommand{\abs}[1]{\vert {#1} \vert}

\newcommand{\diff}{\mathrm{d}}
\newcommand{\idiff}{\,\diff}

\newcommand{\Expect}{\operatorname{\mathbb{E}}}

\newcommand{\Probe}{\mathbb{P}}

%% file: packages.tex
\usepackage{fontawesome}

\usepackage{subfigure}

\usepackage{subfigure}

\usepackage{hyperref}
\usepackage{xcolor}
\hypersetup{
    colorlinks,
    linkcolor={blue!50!black},
    citecolor={blue!50!black},
    urlcolor={blue!50!black}
}

\usepackage[utf8]{inputenc} %
\usepackage[T1]{fontenc}    %
\usepackage{natbib}
\usepackage{url}            %
\usepackage{booktabs,array} %
\usepackage{multirow}
\usepackage{nicefrac}       %
\usepackage[final]{microtype}      %
\usepackage[final]{graphicx}
\usepackage{svg}
\usepackage{tcolorbox}
\usepackage{enumitem}
\usepackage{amsfonts,amssymb,amsthm,amsmath}
\usepackage{mathtools}
\usepackage{bm}
\usepackage{mathabx}
\usepackage{algorithm,algorithmic}
\renewcommand\algorithmiccomment[1]{$\triangleright$ {#1}}

\usepackage{thmtools,thm-restate}
\usepackage{resizegather}
\usepackage{tikz}
\usetikzlibrary{tikzmark,calc}
  {
    \theoremstyle{definition}
    
      \theoremstyle{plain}

      \newtheorem{theorem}{Theorem}
      
      \newtheorem{lemma}{Lemma}
  }

\newcommand*\circled[1]{%
  \tikz[baseline=(C.base)]\node[draw,circle,inner sep=1.2pt,line width=0.2mm,](C) {#1};}
\newcommand*\Myitem{%
  \stepcounter{enumi}\item[\circled{\theenumi}]}

%% file: section/00-abstract.tex
\begin{abstract}
    A \textit{Bayes point machine} is a single classifier that approximates the majority decision of an ensemble of classifiers. This paper observes that kernel interpolation is a Bayes point machine for Gaussian process classification. This observation facilitates the transfer of results from both \textit{ensemble theory} as well as an area of convex geometry known as \textit{Brunn-Minkowski theory} to derive PAC-Bayes risk bounds for kernel interpolation. Since large margin, infinite width neural networks are kernel interpolators, the paper's findings may help to explain generalisation in neural networks more broadly. Supporting this idea, the paper finds evidence that large margin, finite width neural networks behave like Bayes point machines too.
    \begin{center}
        {\Large\faGithub}~~\raisebox{0.45ex}{\url{https://github.com/jxbz/bpm}.}
    \end{center}
\end{abstract}

%% file: section/01-intro.tex
\section{Introduction}
\label{sec:intro}

The ability of a learner to perfectly interpolate the examples from their teacher and still generalise is a mystery of modern machine learning research \citep{benign}. Such \textit{benign overfitting} is visible in neural networks (NNs) trained to zero loss, Gaussian process (GP) posterior draws and interpolators in a reproducing kernel Hilbert space (RKHS).

Since kernel regression forms the backbone of GP inference \citep{Kanagawa2018GaussianPA} and GP inference emerges under various approximations to NN training \citep{radford, NEURIPS2018_5a4be1fa}, there may be a single unifying theory to explain benign overfitting. \citet{belkin18a} suggest that kernel interpolation may be the right place to start looking---but see Section \ref{sec:related} for an overview of other approaches.

\subsection{Kernel Interpolation}

In this paper, the \textbf{kernel interpolator} of a training set $X=\{x_1,...,x_n\}$ with labels $\Upsilon\in\R^n$ refers to:
\begin{equation*}\label{eq:kernel-interp}
    f_\Upsilon(x) = K_{xX}K_{XX}^{-1}\Upsilon,
\end{equation*}
for Gram vector $K_{xX}^{(i)}:=k(x,x_i)$, Gram matrix $K_{XX}^{(ij)} := k(x_i,x_j)$ and positive definite \textbf{kernel function} $k(\cdot,\cdot)$. Formally, $f(x)$ is the interpolator of $(X,\Upsilon)$ with minimum RKHS-norm. A binary classification is made via $\sign f(x)$.

\subsection{Knowledge vs.~Belief}

Kernel interpolation can be understood in terms of \textit{signal processing}. Given a training sample, there are infinitely many interpolating \textit{aliases}. If the underlying function is known to be smooth, then it makes sense to return the \textit{smoothest alias} with respect to a kernel \citep{poggio}.

This point of view underlies recent papers that quantify smoothness either in terms of spectral bias and frequency content \citep{pmlr-v119-bordelon20a} or directly via RKHS-norm \citep{make-mistakes}. The drawback of these approaches to studying generalisation is that prior knowledge is needed about the underlying function's  smoothness. %

This general drawback is sidestepped by PAC-Bayes theory, which replaces the need for prior \textit{knowledge} with a need for prior \textit{belief}. While PAC-Bayes bounds are best for learners with accurate prior belief, a degree of generalisation can still be guaranteed when this belief was mistaken \citep{McAllester98}. The results of this paper were born of an effort to transfer PAC-Bayes theory over to kernel interpolation.

\subsection{Ensembles vs.~Points}

There is a major hurdle to transferring PAC-Bayes theory over to the kernel setting. PAC-Bayes bounds hold for ensembles of classifiers where each classifier has an associated prior probability, while kernel interpolation returns a single classifier that is a deterministic function of the training sample and has no intrinsic notion of probability.

With that said, a special deterministic classifier can be extracted from a weighted ensemble. The \textit{Bayes classifier} reports the weighted majority vote over the ensemble, and is known to have very good generalisation behaviour: When the prior over teacher functions is known, the Bayes classifier is \textit{optimal} \citep{Devroye1996APT}. When the prior is misspecified, the Bayes classifier may still strongly outperform random ensemble members \citep{lacasse}.

Given the vast expressivity of learners like NNs and GPs, it is natural to ask: \textit{can a learner approximate their own Bayes classifier?} In other words, can a single classifier inherit the favourable properties of a diverse ensemble? This would clearly yield major savings in terms of both memory and compute. In learning theory, such an economical classifier is known as a \textit{Bayes point machine} \citep{bpms}.

This paper observes that \textit{kernel interpolation is a Bayes point machine for GP classification}. An analogous statement is well known in the context of regression: \textit{kernel interpolation is the mean of a GP regression posterior}. But the treatment of classification in this paper is more subtle, enabling the transfer of advanced results in voting theory over to kernel interpolation, making progress on an open ``dilemma'' raised by \citet{seeger} and shedding new light on the behaviour of NNs trained to large normalised margin.

\subsection{Contributions}

This paper both advances the underlying theory of Bayes point machines (BPMs), and also derives specific results relevant to GP, NN and kernel classification.

The paper makes two contributions to the theory of BPMs:
\begin{enumerate}[leftmargin=0.875cm,itemsep=0pt,topsep=0pt]
    \item[\S~\ref{sec:pessimist}] BPMs are often derived via a \textit{trick} that approximates the Bayes classifier by an ensemble's centre-of-mass \citep{herbrich_book}. Using a tool from convex geometry, this paper makes that trick more rigorous, showing that under mild conditions the centre-of-mass errs at no more than $\econst$ times the Gibbs error.
    \item[\S~\ref{sec:optimist}] BPMs are usually motivated by noting that the Bayes classifier errs at no more than twice the Gibbs error. This paper applies \textit{the $\mathcal{C}$-bound} \citep{lacasse} to demonstrate a situation where a BPM can greatly outperform the Gibbs error. This constitutes progress on an open problem \citep[Section 5.1]{seeger}.
\end{enumerate}

As for GPs, NNs and kernel methods, the paper shows that:
\begin{enumerate}[leftmargin=0.875cm,itemsep=0pt,topsep=0pt]
    \setcounter{enumi}{2}
    \item[\S~\ref{sec:k-bpm}] Kernel interpolators of both the \textit{centroid} and \textit{centre-of-mass} labels can be derived as the BPM of a GP classifier. The centre-of-mass interpolator is harder to compute but turns out to enjoy a smaller risk bound.
    \item[\S~\ref{sec:nn-bpm}] Large margin, infinite width NNs concentrate on kernel interpolators, so these models are BPMs too.
\end{enumerate}

Combining all of these insights, the paper:
\begin{enumerate}[leftmargin=0.875cm,itemsep=0pt,topsep=0pt]
    \setcounter{enumi}{4}
    \item[\S~\ref{sec:pac-bayes}] Derives PAC-Bayes risk bounds for both kernel interpolation and large margin, infinite width NNs. These PAC-Bayes bounds are found to be significantly tighter than a corresponding Rademacher bound \citep[Theorem 21]{rademacher}.
\end{enumerate}

Finally, on the experimental side, the paper finds that:
\begin{enumerate}[leftmargin=0.875cm,itemsep=0pt,topsep=0pt]
    \setcounter{enumi}{5}
    \item[\S~\ref{sec:expt-gp-k}] The centroidal kernel interpolator attains almost exactly the same error as a GP's Bayes classifier, supporting the claim that kernel interpolation is a BPM.
    \item[\S~\ref{sec:expt-nn}] Finite width, large margin multi-layer perceptrons (MLPs) closely match the majority vote of many small margin MLPs. This suggests that large margin, finite width NNs may also be modelled as BPMs.
\end{enumerate}

%% file: section/02-related.tex
\section{Related Work}
\label{sec:related}

\textbf{Benign overfitting.} A learner's ability to interpolate and still generalise has been studied in linear methods \citep{benign,Chatterji2020FinitesampleAO}, and also in nonlinear methods via a connection between smooth interpolation and overparameterisation \citep{bubeck2021a}. It is also studied in NNs in the context of \textit{double descent} \citep{opper2001learning,Nakkiran2020Deep}. The promise of focusing on \textit{kernel interpolation} is that results may then be transferred directly to GPs \textit{and} infinite width NNs as well.

\textbf{Kernel interpolation.} Many links have been made between NNs, GPs and kernels \citep{radford,choandsaul,NEURIPS2018_5a4be1fa,Kanagawa2018GaussianPA}. Classic studies into the risk of kernel interpolation have worked via Rademacher complexity analysis \citep{rademacher}. More recent studies employ a teacher-student framework \citep{pmlr-v119-bordelon20a} or leverage information about the teacher's RKHS norm \citep{just-interpolate,make-mistakes}.

\textbf{PAC-Bayes theory.} Risk bounds derived via PAC-Bayes analysis \citep{ShaweTaylor1997APA,McAllester98} usually hold for ensembles of classifiers, including GP posteriors \citep{seeger} and stochastic NNs \citep{DR17}. Bounds for individual classifiers have been obtained via margin-based derandomisation \citep{NIPS2002_68d30981} both for support vector machines (SVMs) \citep{svm-pac-bayes} and NNs \citep{neyshabur2018a,Biggs2021OnMA}. Bounds have also been derived that hold individually for most of the ensemble \citep{rivasplataKSS20,Viallard2021AGF} and for mixtures of classifiers \citep{meirzhang,lacasse}.

\textbf{Bayes point machines.} An ensemble member that approximates the Bayes classifier is known as a Bayes point machine \citep{bpms}. Some papers approximate the Bayes classifier via the ensemble centre-of-mass in weight space \citep{billiards,bayeskc}. This approximation is exact under strong symmetry assumptions, and is more generally referred to as a \textit{trick} \citep{herbrich_book}. For perceptron learning, this approximation has been studied from the statistical mechanics perspective \citep{watkin}.

\textbf{Social choice theory.} While voting has been studied in machine learning in the context of boosting \citep{boosting}, bagging \citep{bagging} and Bayes classification \citep{Devroye1996APT}, voting is also studied in the design of democratic systems robust to paradox \citep{arrow}. For instance, how does one settle an election given the preferences of an electorate, while avoiding the \textit{voting paradox of \citet{condorcet}}? One solution is to return the \textit{Simpson-Kramer min-max point} \citep{simpson,Kramer1977ADM}, which is roughly the \textit{least widely disliked} platform. This point is closely related to the \textit{Tukey depth} in statistics \citep{Tukey1977ExploratoryDA} and, as this paper shows, to the Bayes point machine.

%% file: section/03-bpmtheory.tex
\section{Theory of the Bayes Point}
\label{sec:theory}
\begin{quote}
    \textit{Is the halfspace half empty or half full?}
\end{quote}

This section establishes basic definitions and also derives two elementary results about the Bayes point machine---referred to as \textit{Gibbs--BPM} lemmas. The first result is used in Section \ref{sec:pac-bayes} to extend PAC-Bayes risk bounds to kernel methods. The second addresses an open problem \citep{seeger} and offers a jumping-off point for future work. Each result directly mirrors a corresponding \textit{Gibbs--Bayes} lemma.

Consider a classifier $f : \mathcal{X}\times\mathcal{W}\to\R$, where $\mathcal{X}$ is the input space, $\mathcal{W}$ is the weight space and the binary decision is made by $\sign\circ f$. An \textit{ensemble of classifiers} shall be specified via a probability measure $Q$ over weight space $\mathcal{W}$. It will often make sense to think of and refer to $Q$ as a \textbf{posterior distribution}, since the paper will often restrict the support of $Q$ to the \textbf{version space} of a learning problem---meaning the subset of classifiers that correctly classify the training set. The ensuing theory could be generalised to a definition of version space that includes all classifiers that classify the training set to above, say, $95\%$ accuracy.

Given a posterior distribution $Q$, this section will consider classifying a fresh input $x\in\mathcal{X}$ in one of three ways:
\begin{enumerate}[itemsep=0pt,topsep=0pt]
    \item The \textbf{Gibbs classifier} returns a random prediction:
        \begin{equation*}
            \sign f(x;w), \qquad w\sim Q.
        \end{equation*}
    \item The \textbf{Bayes classifier} returns the majority vote:
        \begin{equation*}
            f_\mathrm{Bayes}(x) := \sign \Expect_{w\sim Q} \sign f(x;w).
        \end{equation*}
    \item The \textbf{BPM classifier} returns the simple average:
        \begin{equation*}
            f_\mathrm{BPM}(x) := \sign \Expect_{w\sim Q} f(x; w).
        \end{equation*}
\end{enumerate}

Two observations motivate the BPM classifier. First, it is obtained by reversing $\sign$ and $\Expect_{w\sim Q}$ in the Bayes classifier:
\begin{equation*}\label{eq:approx}
\underbrace{\sign \tikzmarknode{a}{\Expect}_{w\sim Q} \,\tikzmarknode{b}{\sign} f(x;w)}_{\text{Bayes classifier}} \approx \underbrace{\sign \Expect_{w\sim Q}f(x;w)}_{\text{BPM classifier}}.
\end{equation*}\tikz[remember picture, overlay]{\draw[latex-latex] ([yshift=0.15em,xshift=0.5em]a.north) to[bend left] ([yshift=0.15em,xshift=-0.2em]b.north);}%
This operator exchange is called the \textbf{BPM approximation}, the quality of which will be considered in detail in this paper. Second, suppose the classifier has \textit{hidden linearity}. In particular, consider classifier $f_{\mathrm{\phi}}(x;w) := \phi(x)^Tw$, where $\phi$ is an arbitrary nonlinear input embedding. Then:
\begin{equation*}\label{eq:linear}
\underbrace{\sign \Expect_{w\sim Q}f_\phi(x;w)}_{\text{BPM classifier}} = \sign f_\phi(x;\underbrace{\Expect_{w\sim Q}w}_{\mathclap{\text{centre-of-mass}}}).
\end{equation*}
In words: the BPM classifier is equivalent to a single classifier that uses the posterior $Q$'s centre-of-mass for weights. Therefore, by the BPM approximation, a linear classifier's BPM is a \textit{point approximation} to the Bayes classifier.

Of course, since the $\sign$ function is nonlinear, the BPM approximation is not correct in general---\citet{herbrich_book} calls it a \textit{trick}. In the case of hidden linearity (Equation \ref{eq:linear}), the approximation is correct when over half the ensemble $w\sim Q$ agrees with the centre-of-mass $\Expect_{w\sim Q}w$ on input $x$. This happens, for example, when the posterior $Q$ is point symmetric about the centre-of-mass \citep{herbrich_book}. But point symmetry is a strong assumption that does not hold for, say, the version space of a GP classifier.

The next two subsections rigorously connect the risk of the BPM classifier to the risk of the Gibbs and Bayes classifiers. Subsection \ref{sec:pessimist} employs an advanced tool from convex geometry to show that, under mild conditions, a linear classifier's BPM cannot perform substantially worse than the Gibbs classifier. Subsection \ref{sec:optimist} presents a more optimistic perspective, demonstrating a setting where the BPM classifier substantially outperforms the Gibbs classifier.

But first, it is useful to define the \textbf{Gibbs error}, the \textbf{Bayes error} and the \textbf{BPM error}. These notions of error each depend on a data distribution $\mathcal{D}$ over $\mathcal{X}\times\{\pm1\}$:
\begin{align*}
    \gibbserr &:= \Expect_{w\sim Q} &&\hspace{-2em}\Expect_{(x,y)\sim\mathcal{D}}\mathbb{I}\big[\sign f(x;w)\neq y\big]; \\
    \bayeserr &:= &&\hspace{-2em}\Expect_{(x,y)\sim\mathcal{D}}\mathbb{I}\big[f_\mathrm{Bayes}(x)\neq y\big]; \\
    \bpmerr &:= &&\hspace{-2em}\Expect_{(x,y)\sim\mathcal{D}}\mathbb{I}\big[f_\mathrm{BPM}(x)\neq y\big].
\end{align*}

\subsection{A Pessimistic Gibbs--BPM Lemma}
\label{sec:pessimist}

It is well known that the Bayes classifier cannot err at more than twice the Gibbs error \citep{herbrich_book}:

\begin{lemma}[Pessimistic Gibbs--Bayes]\label{lem:gibbs-bayes} For any ensemble,
\begin{equation*}
    \bayeserr \leq 2 \cdot \gibbserr.
\end{equation*}
\end{lemma}
This result is tagged \textit{pessimistic} since one generally hopes for the Bayes classifier to match or outperform the Gibbs classifier---the idea being that the majority vote should \textit{smooth out} the variance of the Gibbs classifier. 

The paper will now derive an analogous relation between the BPM and Gibbs error. The main idea is that, in the case of \textit{hidden linearity}, it is impossible for too significant a fraction of the ensemble to disagree with the centre-of-mass. This is made rigorous by an elegant result from convex geometry:

\begin{lemma}[Weighted Gr\"unbaum's inequality]\label{lem:grunbaum}
Let $Q$ be a log-concave probability density supported on a convex subset of $\R^d$ with positive volume. Then for any $v\in\R^d$:
\begin{equation*}
\Probe_{w\sim Q} \left[\sign [v^Tw] = \sign [v^T \Expect_{w^\prime\sim Q}w^\prime]\right] \geq 1/e.
\end{equation*}
\end{lemma}
In words: if the sample space of a log-concave density is cut by any hyperplane, then the halfspace containing the centre-of-mass contains at least $1/e \approx 36\%$ of the probability~mass.

This result is a generalisation of \citet{grunbaum}'s inequality due to \citet{meanvoter}. The proof leverages an advanced result in Brunn-Minkowski theory known as the \textit{Pr\'ekopa-Borell theorem}. The result was derived in the context of \textit{social choice theory} to bound the proportion of an electorate with linear preferences that can disagree with the mean voter. But as the following lemma shows, it may also be used to bound the error of a Bayes point machine.

\begin{tcolorbox}[boxsep=0pt, arc=0pt,
    boxrule=0.5pt,
 colback=white]
\begin{lemma}[Pessimistic Gibbs--BPM]\label{lem:gibbs-bpm} Consider an ensemble of classifiers whose distribution at all inputs $x\in\mathcal{X}$ follows:
\begin{equation*}
    f_\phi(x;w) = w^T \phi(x), \qquad w\sim Q.
\end{equation*}
Here $\phi$ is an arbitrary nonlinear input embedding, and $Q$ is a log-concave probability density supported on a convex subset of $\R^d$ with positive volume. Then:
\begin{equation*}
    \bpmerr \leq \econst \cdot \gibbserr.
\end{equation*}
\end{lemma}
\end{tcolorbox}

The lemma's conditions are fairly mild, holding for both kernel and GP classifiers with Gaussian (or truncated Gaussian) posteriors. This enables the lemma's use in Section \ref{sec:pac-bayes}. The result is similar in form to Lemma \ref{lem:gibbs-bayes}, and is also tagged \textit{pessimistic}. This is because its proof uses the fact that the centre-of-mass is always found in the halfspace containing at least a fraction $1/e$ of the mass. The optimist would hope to find the centre-of-mass in the heavier halfspace.

\begin{proof}[Proof of Lemma \ref{lem:gibbs-bpm}]
    First, consider the BPM and Gibbs error on a single datapoint $(x,y)$:
    \begin{align*}
    \bpmerr(x,y) &:= \mathbb{I}\left[\sign \Expect_{w\sim Q}w^T \phi(x)\neq y\right]; \\
    \gibbserr(x,y)&:=\Expect_{w\sim Q} \mathbb{I}\left[\sign w^T \phi(x)\neq y\right].
    \end{align*}
    When the BPM classifier is correct, $\bpmerr(x,y)=0$. When the BPM classifier is incorrect, $\bpmerr(x,y)=1$ and $\gibbserr(x,y) \geq 1/e$ by Lemma \ref{lem:grunbaum}. In either case:
    \begin{align*}
        \bpmerr(x,y)\leq \econst\cdot \gibbserr(x,y).
    \end{align*}
    Taking the average over $(x,y)\sim\mathcal{D}$ yields the result.
\end{proof}

\subsection{An Optimistic Gibbs--BPM Lemma}
\label{sec:optimist}

The optimist would expect the BPM approximation to be \textit{good}---that it should hold for \textit{most inputs}, say. And that the BPM error should not be much worse than the Bayes error. It makes sense to package this optimism into a definition. The \textbf{BPM approximation error} $\Delta$ is given by:
    \begin{equation*}\label{eq:bpm-approx}
        \Delta := \Expect_{(x,y)\sim\mathcal{D}}\mathbb{I}\big[f_\mathrm{BPM}(x)\neq f_\mathrm{Bayes}(x)\big].
    \end{equation*}
So $\Delta$ measures the proportion of inputs for which the BPM approximation fails, and the optimist would expect $\Delta$ to be small. This definition leads directly to the following lemma:
\begin{lemma}[Bayes--BPM]\label{lem:bayes-bpm}
    \begin{equation*}
        \bpmerr \leq \bayeserr + \Delta.
    \end{equation*}
\end{lemma}

\begin{proof}
    First consider the BPM error, Gibbs error and BPM approximation error on a single datapoint $(x,y)$:
    \begin{align*}
    \bpmerr(x,y) &:= \mathbb{I}\left[f_\mathrm{BPM}(x)\neq y\right]; \\
    \bayeserr(x,y)&:= \mathbb{I}\left[f_\mathrm{Bayes}(x)\neq y\right] ;\\
    \Delta(x,y) &:= \mathbb{I}\big[f_\mathrm{BPM}(x)\neq f_\mathrm{Bayes}(x)\big].
    \end{align*}
    When the BPM classifier is correct, $\bpmerr(x,y)=0$. When the BPM clssifier is incorrect, $\bpmerr(x,y)=1$ and either $\bayeserr(x,y)=1$ and $\Delta(x,y)=0$ or vice versa. Thus:
    \begin{align*}
    \bpmerr(x,y) \leq \bayeserr(x,y) + \Delta(x,y).
    \end{align*}
    Taking the average over $(x,y)\sim\mathcal{D}$ yields the result.
\end{proof}

In the spirit of continued optimism, one may expect the Bayes classifier to outperform the Gibbs classifier. This is because the majority vote is intended to \textit{smooth out} the variance of the Gibbs classifier. This idea has been formalised via the $\mathcal{C}$-bound \citep{lacasse,germain}:

\begin{lemma}[Optimistic Gibbs--Bayes, a.k.a.~the $\mathcal{C}$-bound]\label{lem:gibbs-bayes-opt} Let $\alpha_{\mathrm{Gibbs}}\in[0,1]$ denote the average Gibbs agreement:
\begin{equation*}
\alpha_{\mathrm{Gibbs}}:=\Expect_{x\sim\mathcal{D}}\left[\left[\Expect_{w\sim Q}\sign f(x;w)\right]^2\right].
\end{equation*}
Then the Bayes error satisfies:
\begin{equation*}
    \bayeserr \leq 1 - \frac{(1-2 \cdot \gibbserr)^2}{\alpha_{\mathrm{Gibbs}}}.
\end{equation*}
\end{lemma}
Lemma \ref{lem:gibbs-bayes-opt} is capable of certifying that $\bayeserr \ll \gibbserr$. This happens when the Gibbs classifier is very noisy, such that the Gibbs error falls just below one half and the average Gibbs agreement $\alpha_{\mathrm{Gibbs}}$ is small.

Combining Lemmas \ref{lem:bayes-bpm} and \ref{lem:gibbs-bayes-opt} yields the following lemma:

\begin{tcolorbox}[boxsep=0pt, arc=0pt, boxrule=0.5pt, colback=white]
\begin{lemma}[Optimistic Gibbs--BPM]\label{lem:gibbs-bpm-opt}
    Let $\alpha_{\mathrm{Gibbs}}$ denote the average Gibbs agreement as in Lemma \ref{lem:gibbs-bayes-opt}, and let $\Delta$ denote the BPM approximation error. Then:
    \begin{equation*}
        \bpmerr \leq 1 - \frac{(1-2 \cdot \gibbserr)^2}{\alpha_{\mathrm{Gibbs}}} + \Delta.
    \end{equation*}
\end{lemma}
\end{tcolorbox}
This result implies that under reasonable conditions---when $\alpha_{\mathrm{Gibbs}}$ and $\Delta$ are both very small---the BPM classifier can substantially outperform the Gibbs classifier. This provides a crisp theoretical motivation for the significance of the Bayes point machine, addressing an open problem \citep[Section 5.1]{seeger}. While Lemma \ref{lem:gibbs-bpm-opt} is not explored further in this paper, the authors believe that this result presents an exciting jumping-off point for future work.

%% file: section/04-nn-gp-k.tex
\section{NNs, GPs and Kernel Interpolators}
\label{sec:nn-gp-k}

This section establishes two main results: first, the sign of a kernel interpolator is the BPM of a GP classifier. And second, at large margin the function space of an infinite width NN concentrates on a kernel interpolator. Taken together, these results imply that margin maximisation (or, dually, weight norm minimisation) converts an infinite width NN into a BPM---as illustrated schematically in Figure \ref{fig:games}.

\subsection{Kernel Interpolation is a Bayes Point Machine}
\label{sec:k-bpm}

Consider a GP with covariance function $k(\cdot,\cdot)$, a set of $n$ training points $X=\{x_1,...,x_n\}$ and a vector of $n$ binary labels $Y=[y_1,...,y_n]$. It is useful to define the Gram vector $K_{xX}^{(i)}:=k(x,x_i)$ and Gram matrix $K_{XX}^{(ij)} := k(x_i,x_j)$.

The paper constructs a \textbf{GP Gibbs classifier} by sampling functions from the GP prior and rejecting those functions with incorrect sign on the training points. Formally, this corresponds to a GP posterior with zero--one likelihood. Predictions at a test point $x$ may be generated in three steps:
\begin{align*}
\text{Sample labels:}\quad &\Upsilon \sim \mathcal{N}\big(0,K_{XX} | \sign \Upsilon = Y \big);\\
\text{Sample noise:}\quad &\xi \sim \mathcal{N}\left(0, K_{xx} - K_{xX}K_{XX}^{-1} K_{Xx}\right);\\
\text{Return:}\quad &\sign [K_{xX} K_{XX}^{-1} \Upsilon + \xi ].
\end{align*}
The corresponding \textbf{GP BPM classifier} is then obtained by exchanging operators in the \textbf{GP Bayes classifier}:
\begin{align*}
	&\sign \tikzmarknode{a}{\Expect}_{\xi,\Upsilon}\,\tikzmarknode{b}{\sign} [K_{xX} K_{XX}^{-1} \Upsilon + \xi ] \qquad \tag{Bayes classifier}\\
	&\qquad\approx\sign \Expect_{\xi,\Upsilon} [K_{xX} K_{XX}^{-1} \Upsilon + \xi ]\tag{BPM classifier}\\
	&\qquad=\sign [K_{xX} K_{XX}^{-1} Y_{\com} ]. \tag{kernel interpolator}
\end{align*}
\tikz[remember picture, overlay]{\draw[latex-latex] ([yshift=0.15em,xshift=0.5em]a.north) to[bend left] ([yshift=0.15em,xshift=-0.2em]b.north);}%
So the GP BPM classifier is equivalent to the sign of the kernel interpolator with \textbf{centre-of-mass labels} $Y_{\com} := \Expect_\Upsilon \Upsilon$.

For reasons of both analytical and computational tractability, it is also convenient to modify the GP posterior to employ an isotropic distribution over training labels. The \textbf{isotropic Gibbs classifier} classifies a fresh point $x$ in three steps:
\begin{align*}
\text{Sample labels:}\quad &\Upsilon \sim \mathcal{N}\big(0,\mathbb{I}\cdot \abs{K_{XX}}^{1/n} | \sign \Upsilon = Y \big);\\
\text{Sample noise:}\quad &\xi \sim \mathcal{N}\left(0, K_{xx} - K_{xX}K_{XX}^{-1} K_{Xx}\right);\\
\text{Return:}\quad &\sign [K_{xX} K_{XX}^{-1} \Upsilon + \xi ].
\end{align*}
The corresponding \textbf{isotropic BPM classifier} is obtained by exchanging operators in the \textbf{isotropic Bayes classifier}:
\begin{align*}
	&\sign \tikzmarknode{a}{\Expect}_{\xi,\Upsilon}\,\tikzmarknode{b}{\sign} [K_{xX} K_{XX}^{-1} \Upsilon + \xi ]\qquad \tag{Bayes classifier}\\
	&\qquad\approx\sign \Expect_{\xi,\Upsilon} [K_{xX} K_{XX}^{-1} \Upsilon + \xi ]\tag{BPM classifier}\\
	&\qquad=\sign [K_{xX} K_{XX}^{-1} Y ]. \tag{kernel interpolator}
\end{align*}
\tikz[remember picture, overlay]{\draw[latex-latex] ([yshift=0.15em,xshift=0.5em]a.north) to[bend left] ([yshift=0.15em,xshift=-0.2em]b.north);}%
So the istropic BPM classifier is nothing but the sign of the kernel interpolator with \textbf{centroidal labels} Y.

In Section \ref{sec:pac-bayes}, it turns out that the centre-of-mass interpolator enjoys a smaller risk bound than the centroidal interpolator.

\subsection{Infinite Width NNs as Bayes Point Machines}
\label{sec:nn-bpm}
\input{figure/games}

Consider an $L$-layer multi-layer perceptron $f_{L} : \mathcal{X}\times\mathcal{W}\to\R$ with weight matrices $w=(W_1,...,W_L)\in\mathcal{W}$ and nonlinearity set to $\mathrm{relu}(\cdot):=\max(0,\cdot)$:
\begin{equation*}
    f_L(x;w):= W_L \circ \mathrm{relu} \circ W_{L-1} \circ ... \circ \mathrm{relu} \circ W_1 (x).
\end{equation*}
A \textbf{prior distribution} $P_\sigma$ over weight space $\mathcal{W}$ is constructed by sampling each weight entry iid $\mathcal{N}(0,\sigma^2)$. This induces a prior over functions with mean and covariance given by:
\begin{align*}
    \mu_\sigma(x) &:= \Expect_{w\sim P_\sigma} [f_L(x;w)] = 0; \\
    k_\sigma(x,x^\prime) &:= \Expect_{w\sim P_\sigma} [f_L(x;w)f_L(x^\prime;w)] = \sigma^{2L} k_{1}(x,x^\prime).
\end{align*}
The first equality follows since $\Expect_{w\sim P}{W_L}=0$. The second is due to the degree-$L$ positive homogeneity of the $\mathrm{relu}$ MLP. It follows that, for $n$ training inputs $X=\{x_1,...,x_n\}$ and a test input $x$, the Gram matrix, vector and scalar satisfy:
\begin{equation*}
    K_{XX} = \sigma^{2L}\widehat{K}_{XX}; \; K_{xX} = \sigma^{2L}\widehat{K}_{xX}; \; K_{xx} = \sigma^{2L}\widehat{K}_{xx},
\end{equation*}
where $K$ and $\widehat{K}$ correspond to $k_\sigma$ and $k_1$, respectively.

By the NN--GP correspondence \citep{radford,lee2018deep,g.2018gaussian}, as the MLP width is sent to infinity, the prior over functions converges to a GP with covariance $k_\sigma(x,x^\prime)$. Consider using this GP to classify a test point $x$ by \textit{regressing} to training inputs $X$ and binary labels $Y$. One is free to first scale up the labels by a \textbf{margin} parameter $\gamma>0$---see Figure \ref{fig:games}. Conditioned on interpolating $(X,\gamma Y)$, a posterior prediction is given by: 
{\small\begin{align*}
    & \sign\left(K_{xX}K_{XX}^{-1}(\gamma Y) + \eta\cdot\sqrt{K_{xx}-K_{xX}K_{XX}^{-1}K_{Xx}} \right) \\
    &= \sign\left(\widehat{K}_{xX}\widehat{K}_{XX}^{-1} Y + \frac{\sigma^L}{\gamma}\cdot\eta\cdot\sqrt{\widehat{K}_{xx}-\widehat{K}_{xX}\widehat{K}_{XX}^{-1}\widehat{K}_{Xx}} \right),
\end{align*}}%
for $\eta\sim\mathcal{N}(0,1)$. So by taking the \textbf{normalised margin} $\gamma/\sigma^L \to \infty$, an NN--GP's entire function space in effect \textit{concentrates} on the kernel interpolator $\widehat{K}_{xX}\widehat{K}_{XX}^{-1} Y$, which is itself a Bayes point machine by the results of Section \ref{sec:k-bpm}.

%% file: figure/games.tex
\begin{figure}
    \begin{minipage}{\linewidth}
      \centering
      \raisebox{-0.5\height}{\includegraphics[width=0.4\linewidth]{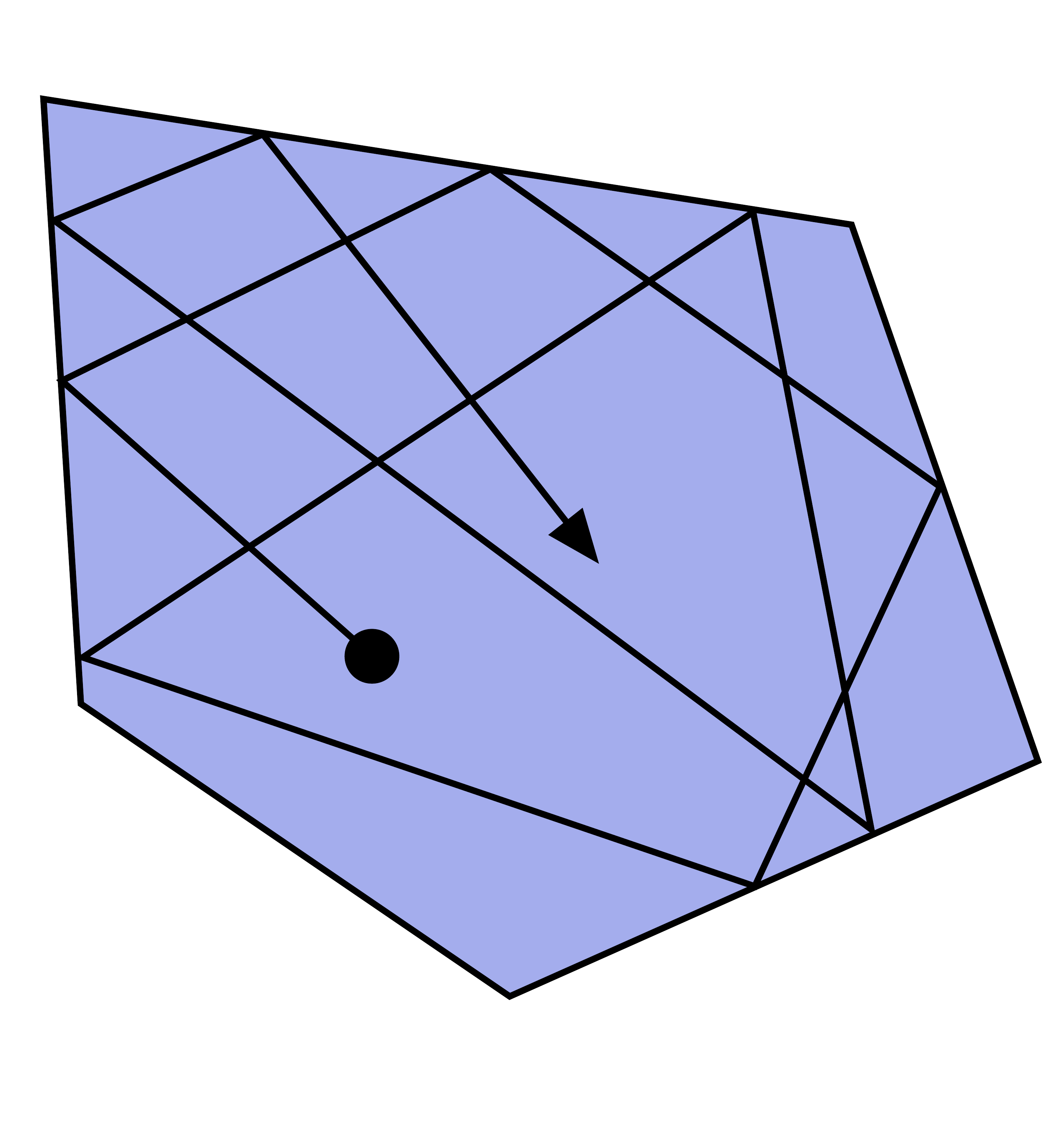}}
      \hspace{.2in}
      \raisebox{-0.5\height}{\includegraphics[width=0.4\linewidth]{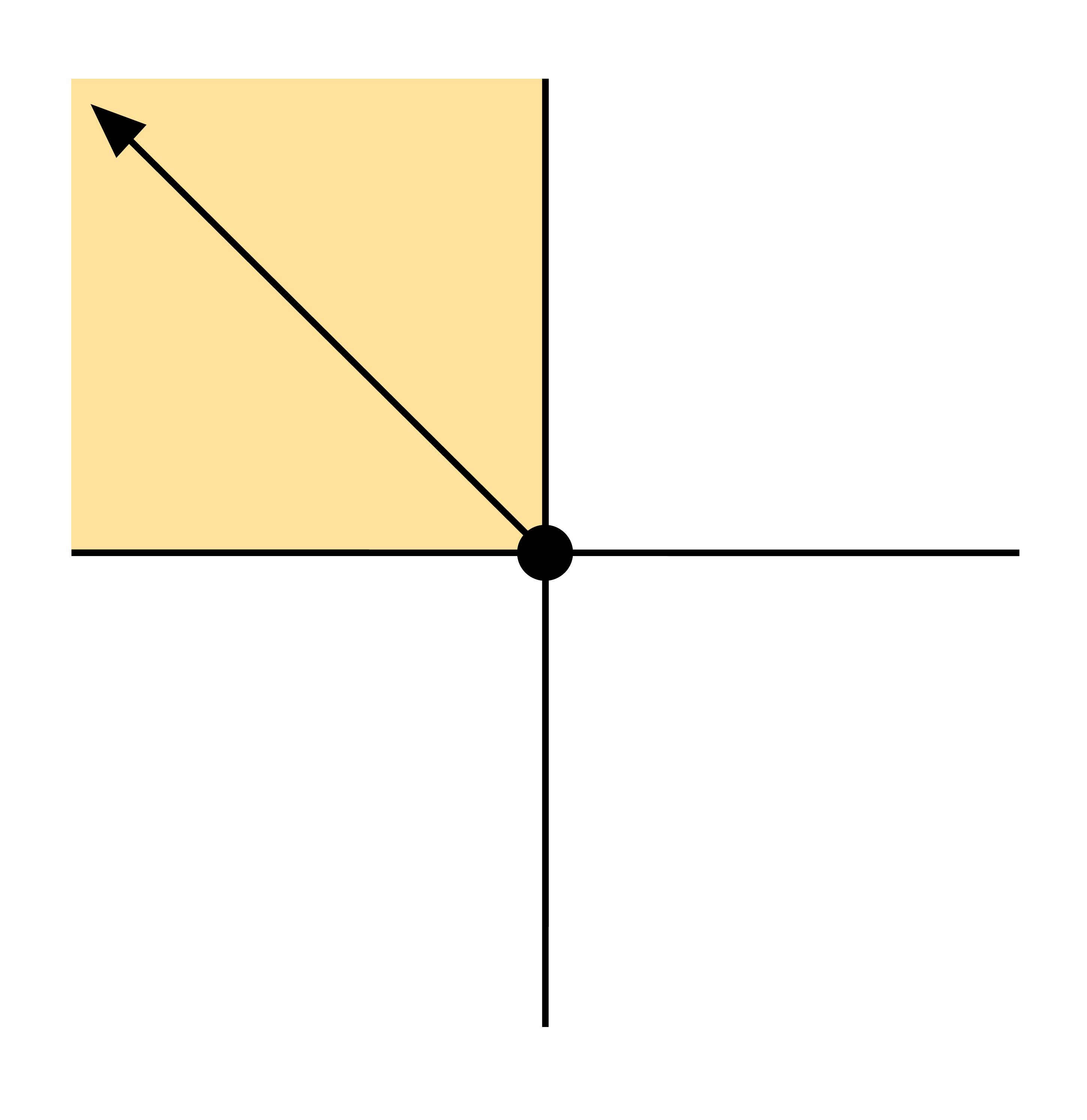}}
    \end{minipage}
    \hspace{2em}
    \vspace{-10pt}
    \caption{Left: In the weight space of a perceptron, the version space is a polytope. \citet{billiards} suggests \textit{playing billiards in weight space}: averaging the ergodic trajectory of a reflected billiard ball to approximate the Bayes classifier. Right: In the function space of an NN binary classifier, where the axes measure the training outputs, the version space is an orthant. This paper suggests \textit{playing baseball in function space}: maximising the normalised margin on the training data to approximate the Bayes classifier.}
    \label{fig:games}
\end{figure}

%% file: section/05-pac-bayes.tex
\section{Kernel PAC-Bayes}
\label{sec:pac-bayes}

This section combines the results from Sections \ref{sec:theory} and \ref{sec:nn-gp-k} with a novel bound on Gaussian orthant probabilities in order to derive risk bounds for kernel interpolators---and also for infinite width NNs by the results of Section \ref{sec:nn-bpm}.

The starting point is a well-known bound on the Gibbs error, that follows from Theorem 3 of \citet{Langford01boundsfor}:

\begin{lemma}[Gibbs PAC-Bayes]\label{lem:pac-bayes}
Let $P$ be a prior over functions realised by a classifier. With probability $1-\delta$ over the $n$ iid training examples drawn from $\mathcal{D}$, for any posterior $Q$ over functions that correctly classify the training data:
\begin{equation*}
    \gibbserr \leq 1 - \exp \left(- \frac{\kl(Q||P) + \log(2n/\delta)}{n-1} \right).
\end{equation*}
\end{lemma}
To apply this lemma to a GP classifier, an appropriate KL divergence is needed. A draw from the GP prior correctly classifies the training data when the sampled labels $\Upsilon$ on the training inputs have the correct sign: $\sign \Upsilon = Y$. Thus, for the purposes of Lemma \ref{lem:pac-bayes}, it is enough to only consider distributions over the training labels. The prior, posterior and approximate posterior are then given by:
\begin{align*}
    P_{\mathrm{GP}}: \qquad & \Upsilon\sim\mathcal{N}\big(0,K_{XX}\big); \hspace{3.15em}\\
    Q_{\mathrm{GP}}: \qquad & \Upsilon\sim\mathcal{N}\big(0,K_{XX} | \sign \Upsilon = Y\big);\\
    Q_{\mathrm{iso}}: \qquad & \Upsilon\sim\mathcal{N}\big(0, \mathbb{I}\cdot \abs{K_{XX}}^{1/n} | \sign \Upsilon = Y\big).
\end{align*}
To derive the corresponding KL divergences, it is first useful to define the \textbf{Gaussian orthant probability} $P_Y$ of the orthant picked out by the binary vector $Y\in\{\pm1\}^n$ via:
\begin{equation*}\label{eq:gop}
    P_Y : = \Probe_{\Upsilon \sim P_{\mathrm{GP}}}[\sign \Upsilon = Y].
\end{equation*}
It is also useful to define another kernel complexity measure:
\begin{align*}
    \mathcal{A}(k,X,Y):= n\cdot(\log 2 - \half) \;+ \qquad\qquad\qquad\qquad\\
    \abs{K_{XX}}^{1/n}\cdot\left[ (\half - \tfrac{1}{\pi})\trace K_{XX}^{-1} + \tfrac{1}{\pi} Y^T K_{XX}^{-1}Y \right]. \nonumber
\end{align*}
This paper then obtains the following exact KL divergences:
\begin{tcolorbox}[boxsep=0pt, arc=0pt,
    boxrule=0.5pt,
 colback=white]
\begin{lemma}[KL divergences on an orthant]\label{lem:gop}
    \begin{align*}
        \kl(Q_{\mathrm{GP}} || P_{\mathrm{GP}}) &= \log (1/P_Y) \\
        &\leq \mathcal{A}(k,X,Y) = \kl(Q_{\mathrm{iso}} || P_{\mathrm{GP}}).
    \end{align*}
\end{lemma}
\end{tcolorbox}
\begin{proof} To establish the first equality, observe that:
\begin{equation*}
    \kl(Q_{\mathrm{GP}} || P_{\mathrm{GP}}) = \Expect_{\Upsilon\sim Q_{\mathrm{GP}}} \log(1/P_Y)= \log (1/P_Y).
\end{equation*}
The last equality is derived by first observing that:
\begin{align*}
    \kl(Q_{\mathrm{iso}} || P_{\mathrm{GP}}) = \Expect_{\Upsilon \sim Q_{\mathrm{iso}}} \log \frac{2^n\cdot\econst^{-\half\Upsilon^2\abs{K_{XX}}^{-1/n}}}{\econst^{-\half\Upsilon^TK_{XX}^{-1}\Upsilon}} \qquad\\
    = n \log 2 + \half \Expect_{\Upsilon \sim \Q_{\mathrm{iso}}} [\Upsilon^T (K_{XX}^{-1} - \mathbb{I}\abs{K_{XX}}^{-1/n})\Upsilon].
\end{align*}
To complete the result, one must substitute in the identity:
\begin{gather*}
    \Expect_{\Upsilon \sim \Q_{\mathrm{iso}}}[\Upsilon_i\Upsilon_j] = \abs{K_{XX}}^{1/n}\cdot\left[\delta_{ij} + \tfrac{2}{\pi} Y_i Y_j (1-\delta_{ij})\right].
\end{gather*}

Finally, the inequality follows via:
\begin{align*}
    \kl(Q_{\mathrm{iso}} || P_{\mathrm{GP}}) &= \Expect_{\Upsilon\sim Q_{\mathrm{iso}}}\left[\log\tfrac{Q_{\mathrm{iso}}(\Upsilon)}{Q_{\mathrm{GP}}(\Upsilon)} + \log\tfrac{Q_{\mathrm{GP}}(\Upsilon)}{P_{\mathrm{GP}}(\Upsilon)}\right]\\
    &= \kl(Q_{\mathrm{iso}}||Q_{\mathrm{GP}}) + \log(1/P_Y),
\end{align*}
and noting that $\kl(Q_{\mathrm{iso}}||Q_{\mathrm{GP}}) \geq 0$.
\end{proof}

By combining Lemmas \ref{lem:gibbs-bpm}, \ref{lem:pac-bayes} and \ref{lem:gop} with the observation in Section \ref{sec:k-bpm} that kernel interpolators are BPMs corresponding to linear classifiers with log-concave posteriors supported on a convex set, the following result is immediate:
 \begin{tcolorbox}[boxsep=0pt, arc=0pt,
    boxrule=0.5pt,
 colback=white]
\begin{theorem}[Kernel PAC-Bayes]\label{thm:k-bpm} Sample $n$ training points $(X,Y)$ iid from $\mathcal{D}$. Construct a kernel interpolator $f_{\Upsilon}(x) := K_{xX} K_{XX}^{-1} \Upsilon$ and consider its risk:
\begin{align*}
    \eps_{\Upsilon} &:= \Expect_{(x,y)\sim\mathcal{D}} \mathbb{I}[\sign f_{\Upsilon}(x) \neq y].
\end{align*}
With probability $1-\delta$ over the training sample:
\begin{align*}
    \eps_{Y} &\leq \econst \cdot \left[1-\exp\left(-\frac{\mathcal{A}(k,X,Y) + \log(2n/\delta)}{n-1}\right)\right].
\end{align*}
Or, for the \textit{centre-of-mass labels} $Y_{\com} := \Expect_{\Upsilon \sim Q_{\mathrm{GP}}} \Upsilon$, with probability $1-\delta$ over the training sample:
\begin{align*}
    \eps_{Y_{\com}} &\leq \econst \cdot \left[ 1-\exp\left(-\frac{\log (1/P_Y) + \log(2n/\delta)}{n-1} \right)\right].
\end{align*}
\end{theorem}
\end{tcolorbox}

Four important remarks are in order:

First, the complexity term $\log(1/P_Y)$ measures the \textit{degree of surprise} experienced upon observing data sample $(X,Y)$ after fixing kernel $k(\cdot,\cdot)$. The smaller the surprise, the smaller the risk bound on $f_{Y_{\com}}(x)$. The bound is non-vacuous when the sample is \textit{sufficiently unsurprising}.

Second, $\mathcal{A}(k,X,Y)\geq \log (1/P_Y)$ by Lemma \ref{lem:gop}, so the centre-of-mass interpolator $f_{Y_{\com}}(x)$ enjoys a smaller risk bound than the centroidal interpolator $f_Y(x)$. If $f_{Y_{\com}}(x)$ can be computed, it may indeed outperform $f_Y(x)$.

Third, the bounds admit a functional analytic interpretation. The term $Y^T K_{XX}^{-1}Y=( K_{XX}^{-1}Y)^TK_{XX}( K_{XX}^{-1}Y)$ appearing in $\mathcal{A}(k,X,Y)$ is the squared RKHS-norm of $f_Y(x)$. Similarly, $\log(1/P_Y)$ is a form of \textit{log-sum-exp} aggregate of the squared RKHS-norm across the version space:
{\small\begin{equation*}
    \log(1/P_Y) = - \log\int\limits_{\mathclap{\sign \Upsilon = Y}}\exp(-\half \|f_\Upsilon\|_{\mathcal{H}}^2) \idiff{\Upsilon} + \half\log(2\pi)^n \abs{K_{XX}}.
\end{equation*}}%
Fourth, the theorem applies (with probability one) to infinitely wide NNs whose normalised margin is sent to infinity ($\gamma/\sigma^L \to \infty$) by the argument given in Section \ref{sec:nn-bpm}.

%% file: figure/gp.tex
\begin{figure*}
    \centering
    \includegraphics[width=0.368\linewidth]{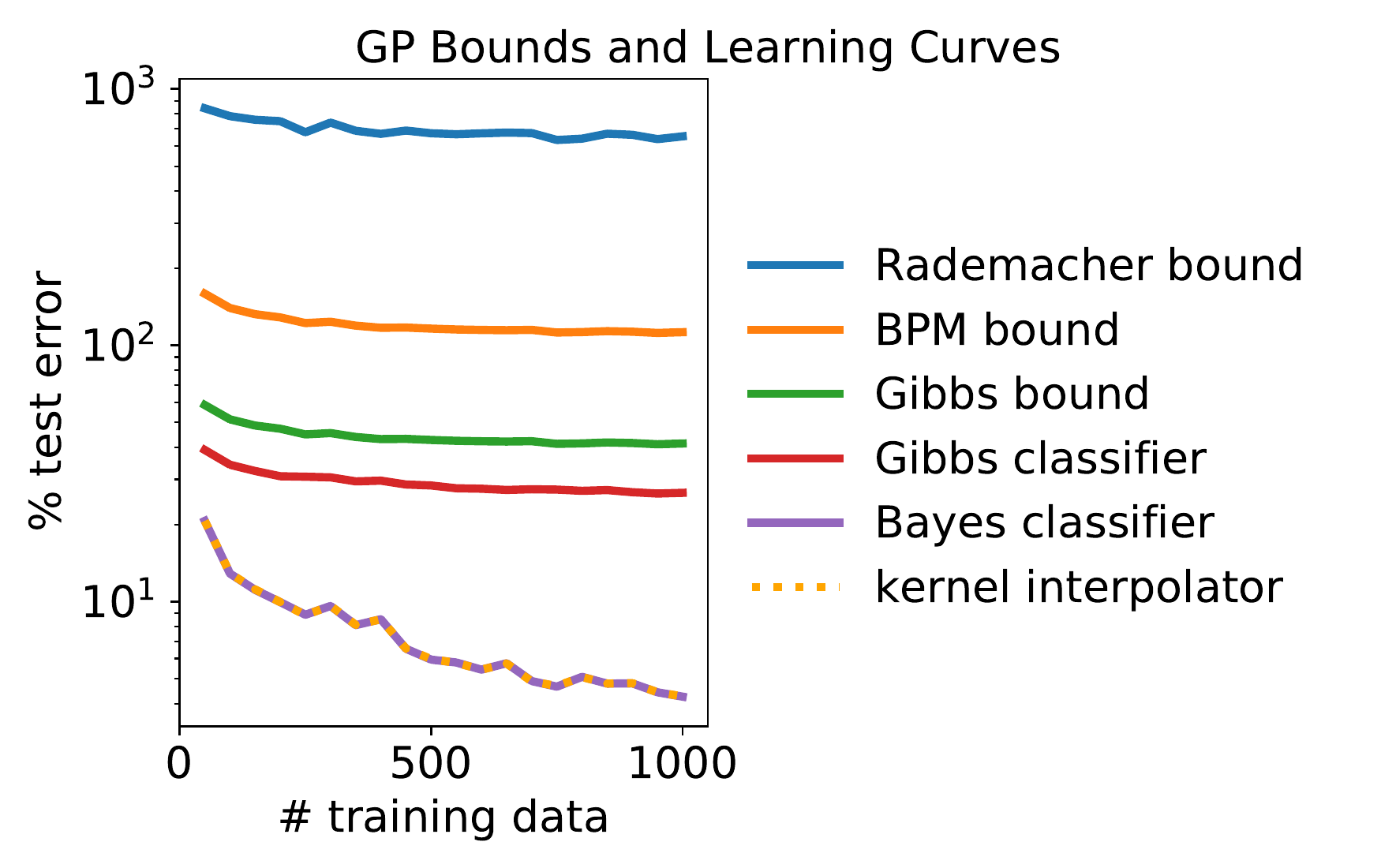}\includegraphics[width=0.315\linewidth]{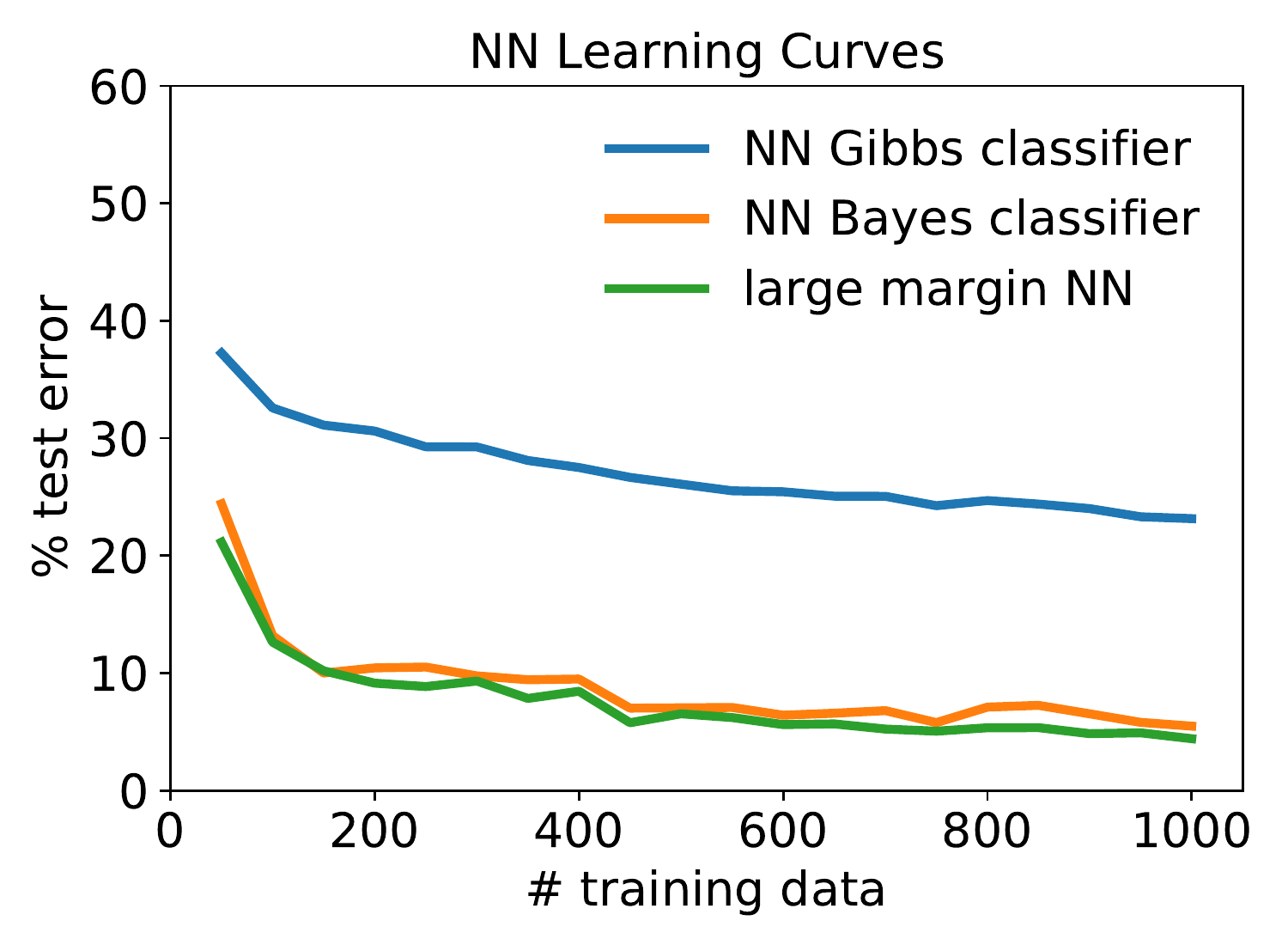}\includegraphics[width=0.315\linewidth]{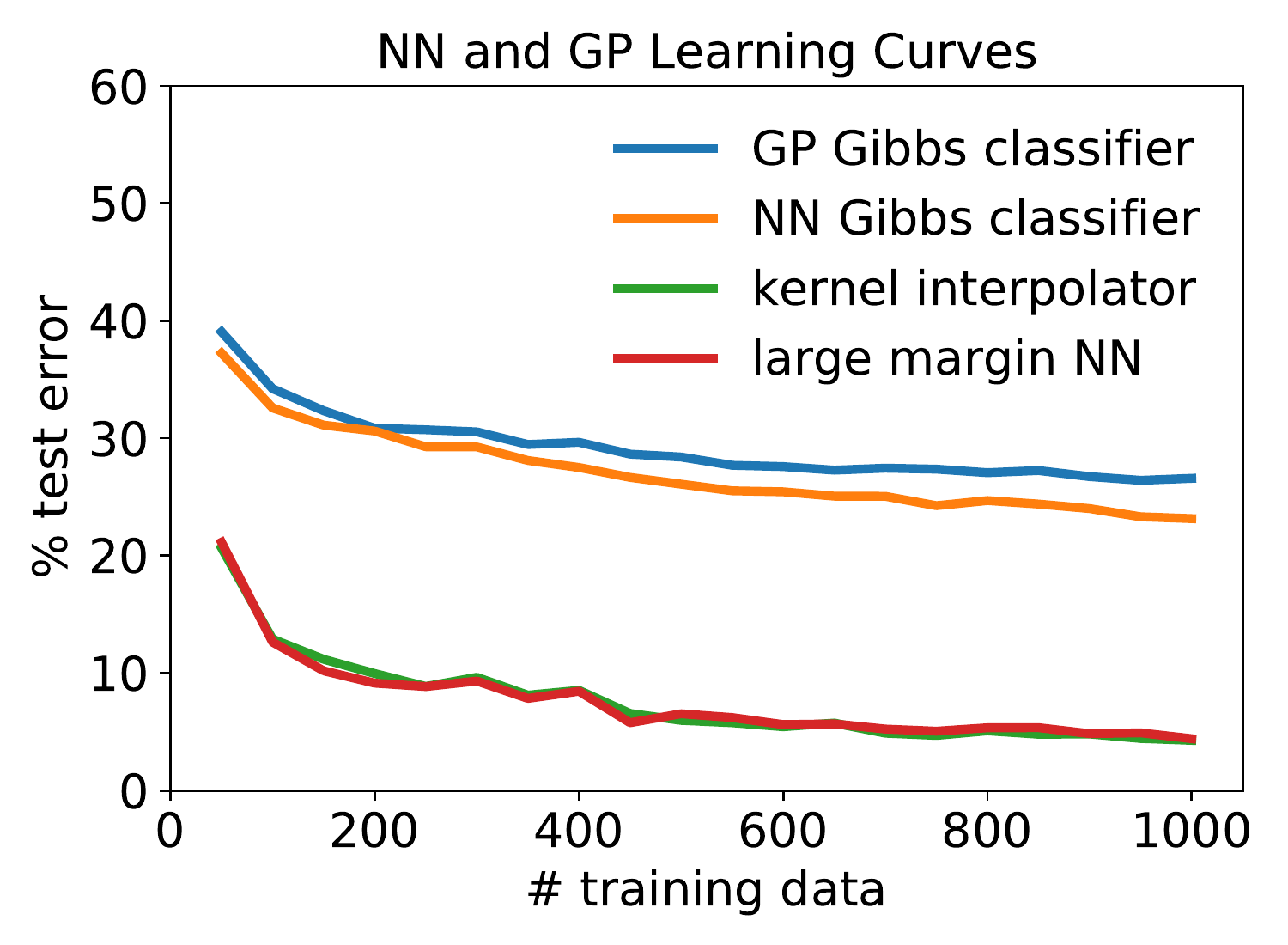}
\caption{\textbf{Left:} Testing the theory for kernel interpolation. This plot serves two purposes: First, it compares the risk bounds in Theorem \ref{thm:k-bpm} and Lemma \ref{lem:pac-bayes} to a classic Rademacher bound of \citet{rademacher}. The Gibbs bound is non-vacuous, while the BPM bound is roughly an order-of-magnitude tighter than the Rademacher bound---although both are vacuous. Second, the plot displays near perfect agreement between the empirical \textit{isotropic Bayes classifier} and the \textit{centroidal kernel interpolator}, validating this paper's central theme. \textbf{Middle:} Testing the theory for finite width multi-layer perceptrons (MLPs). This plot compares the test error of an \textit{ensemble of 501 small-margin MLPs}---each trained to interpolate labels sampled from the label orthant---to a \textit{single large-margin MLP}. The large-margin MLP closely matches the Bayes classifier. \textbf{Right:} Overlaying the kernel and finite width NN results. This plot displays close agreement between a single large-margin MLP and a kernel interpolator that uses the \textit{neural network--Gaussian process} (NN--GP) equivalent kernel. Overall, these results suggest that finite width NNs trained to large margin may indeed approximate their own Bayes classifier.}
    \label{fig:gp}
\end{figure*}

%% file: section/06-experiments.tex
\section{Experiments}
\label{sec:expt}

The purpose of this section is both to test the developed theory and to assess how well it may transfer to finite width NNs. An extensive investigation involving varied datasets and network architectures was beyond this paper's wherewithal and is left to future work. But the authors believe that the results in Figure \ref{fig:gp} are already quite interesting.

The experimental setup involved even/odd classification of \textit{MNIST handwritten digits} \citep{lecun2010mnist}. The GP and kernel experiments used the \textit{compositional arccosine kernel} \citep{choandsaul,daniely,lee2018deep} which is the NN--GP equivalent kernel of an $L$-layer $\mathrm{relu}$ MLP. For two inputs $x,x^\prime \in \R^{d_0}$, it is given by:
\begin{equation*}
    k(x,x^\prime) := \underbrace{h \circ ... \circ h}_{L-1 \text{ times}}\left(\frac{x^T x^\prime}{d_0}\right),
\end{equation*}
where $h(t):=\tfrac{1}{\pi}\cdot [ \sqrt{1-t^2} + t\cdot (\pi - \arccos t)].$ Inputs were normalised to $\|x\|_2^2 = d_0$ so that $k(x,x)=1$.

The experiments resorted to the centroidal kernel interpolator since---despite having better theoretical properties---the centre-of-mass kernel interpolator seems intractable to compute. Correspondingly, the GP experiments used the isotropic Gibbs and Bayes classifiers---see Section \ref{sec:k-bpm}.

The finite width NN experiments used width-1000, depth-7 $\mathrm{relu}$ MLPs trained using the \textit{Nero optimiser} \citep{nero2021}. \textbf{Small margin NNs} were each trained to fit a label vector drawn $\Upsilon \sim Q_{\mathrm{iso}}$ by minimising the loss:
\begin{equation*}
    \mathcal{L}_{\Upsilon}(w) := \sqrt{\frac{1}{n}\sum_i (f(x_i,w) - \Upsilon_i)^2}.
\end{equation*}
The \textbf{large margin NN} was trained to minimise $\mathcal{L}_{100\cdot Y}$.

\subsection{Comparing Bounds}
\label{sec:expt-gp-k}

First the PAC-Bayes bounds of Theorem \ref{thm:k-bpm} and Lemma \ref{lem:pac-bayes} were compared to a Rademacher bound \citep[Theorem 21]{rademacher}, which in this paper's setting says that the risk of the centroidal kernel interpolator obeys:
\begin{align*}
    \eps_Y \leq 4\cdot  \sqrt{Y^T K_{XX}^{-1}Y/n} + \mathrm{confidence\,term}.
\end{align*}
This paper neglected the confidence term, so the Rademacher curve in Figure \ref{fig:gp} (left) is a slight underestimate. Nevertheless, the Rademacher bound was nearly an order-of-magnitude worse than the BPM bound, which was itself a factor $\econst$ worse than the non-vacuous Gibbs bound.

\subsection{Quality of the BPM Approximation}
\label{sec:expt-nn}

To what extent does a Bayes point machine really reflect the majority behaviour of an entire ensemble? For NN--GPs, this question was studied by comparing the test error of the centroidal kernel interpolator to both the Gibbs and Bayes error of an ensemble of $10^5$ GP posterior draws. As can be seen in Figure \ref{fig:gp} (left) the kernel interpolator almost perfectly recovered the error of the Bayes classifier, and both substantially outperformed the Gibbs classifier.

For finite width NNs, the Gibbs and Bayes classifiers were approximated by training an ensemble of 501 small margin NNs. These were then compared to the performance of a single large margin NN. As can be seen in Figure \ref{fig:gp} (middle), the performance of the large margin NN closely matched the test performance of the approximate Bayes classifier, and both outperformed the approximate Gibbs classifier.

Finally, the NN-GP and finite width NN results are overlayed in Figure \ref{fig:gp} (right).

%% file: section/08-discuss.tex
\section{Discussion and Future Work}

This section highlights connections to existing research tracks, and the potential for exciting future work, by extracting three concrete suggestions from the developed theory.

\subsection*{Suggestion \#1: Interpolate the Centre-of-Mass}
Prior work comments on how \textit{arbitrary} it is to classify by interpolating binary labels $Y$, since any labelling $\Upsilon$ in the orthant $\sign \Upsilon = Y$ yields correct training predictions \citep{make-mistakes}. The theory in this paper suggests that interpolating centre-of-mass labels $Y_{\com} := \Expect_{\Upsilon \sim Q_{\mathrm{GP}}} \Upsilon$ is more principled for two reasons: First, Section \ref{sec:k-bpm} show that it more directly approximates the GP Bayes classifier. Second, Theorem \ref{thm:k-bpm} shows that it enjoys a smaller risk bound.

Conceptually, the centre-of-mass kernel interpolator involves using the kernel's prior to re-label the data before fitting. This bears a striking resemblance to the idea of \textit{self-distillation} \citep{Furlanello2018BornAN}, which trains a \textit{teacher} NN on the training data, and then retrains a \textit{student} NN on the teacher's predictions. Self-distillation and related techniques such as \textit{label smoothing} \citep{label-smoothing} and \textit{label mixup} \citep{zhang2018mixup} have been found to improve generalisation performance in practice.

Studying self-distillation as an approximate centre-of-mass BPM could potentially unify and extend prior studies on self-distillation that use both Bayes theory \citep{Zhang2020SelfDistillationAI} and functional analysis \citep{mobahi}.

\subsection*{Suggestion \#2: Maximise the Normalised Margin}
While this suggestion is perhaps unsurprising, the authors feel it may have value in putting standard NN techniques on a more solid footing. Section \ref{sec:nn-bpm} showed that to convert an infinite width NN into a BPM, one must select an interpolator which maximises a quantity $\gamma / \sigma^{L}$. The \textit{margin} $\gamma$ measures the size of the training predictions, and the quantity $\sigma$ measures the scale of the weights at each layer. In NN classification, this motivates one of two strategies:
\begin{enumerate}[,itemsep=0.2pt,topsep=0pt,label=\Alph*]
    \Myitem \begin{minipage}{.95\linewidth}
        Use a margin maximizing loss \citep{rosset2003margin} such as cross-entropy and fix the weight norms;
    \end{minipage}
    \Myitem \begin{minipage}{.95\linewidth}
        Pair a loss function that targets a fixed margin such as mean squared error with $\ell_2$ regularisation. 
    \end{minipage}
\end{enumerate}
Of course, both strategies are in use \citep{hui2021evaluation}. This theoretical suggestion ties in closely with prior work on the implicit bias of optimisation procedures \citep{barrett2021implicit,smith2021on} in their ability to target large margin functions \citep{implicit-bias}.

\subsection*{Suggestion \#3: The Evidence is not Enough}
Prior work on Bayesian model selection suggests choosing a GP kernel \citep{NEURIPS2018_d465f14a} or NN architecture \citep{Prez2020GeneralizationBF,pmlr-v139-immer21a} by maximising the \textit{marginal likelihood} of the data under the model, also known as the \textit{evidence} for the model \citep{mackay}. Due to the difficultly of directly computing the evidence for a high-dimensional model, an \textit{evidence lower bound} (ELBO) is often maximised as proxy \citep{wu2018deterministic}.

The evidence for a GP binary classifier is nothing but the prior probability of the version space $\sign \Upsilon = Y$. This is denoted by the Gaussian orthant probability $P_Y$ in Section \ref{sec:pac-bayes}. This evidence $P_Y$ and its lower bound $\exp - \mathcal{A}(k,X,Y)$ take centre stage in Theorem \ref{thm:k-bpm}, where they are used to upper bound the risk of kernel interpolation. 

But looking at the empirical results in Figure \ref{fig:gp} (left), both the BPM bound and the Gibbs bound are very loose in comparison to the actual performance of both the Bayes classifier and the kernel interpolator. In other words: ELBO maximisation---and also approximate evidence maximisation \citep{Prez2020GeneralizationBF}---is optimising a loose bound on the risk of the most desirable single classifier: the Bayes point machine. This is \citet{seeger}'s ``dilemma''.

A jumping-off point for future work is to explore kernel and NN architecture design by optimising a more optimistic bound on the \textit{Bayes error} such as Lemma \ref{lem:gibbs-bayes-opt}---or more recent alternatives \citep{masegosaLIS20}. The essential implication of Lemma \ref{lem:gibbs-bayes-opt} is that not only should the Gibbs error $\gibbserr$ be minimised, but the version space should also include a sufficient \textit{diversity of opinion} so as to make the Gibbs agreement $\alpha_{\mathrm{Gibbs}}$ small. This language is deliberately evocative of concepts in voter aggregation and social choice \citep{arrow, Kramer1977ADM, sixtyfour, meanvoter}, since it is hoped that more of that literature may be brought to bear upon the learning problem.

\section{Conclusion}

This paper has developed a novel synthesis of ideas and techniques to characterise generalisation in interpolating learning machines. This synthesis draws on the literatures of statistical machine learning, social choice theory and convex geometry. The paper adds to a growing body of work that exploits \textit{hidden convexity} to explain perplexing phenomena in neural networks  \citep{NEURIPS2018_5a4be1fa}. In this paper, it is the convexity of version space when lifted to function space, and the log-concavity of the associated posterior.

At the heart of the paper is an old idea \citep{watkin,billiards} that learners can attempt to point-approximate their own majority classifier. The paper shows how this idea---the \textit{Bayes point machine} \citep{bpms}---may be extended to multi-layer NNs, and how it may be used to derive simple risk bounds for interpolating classifiers. The paper opens up many exciting directions for future work. The authors are curious of where these directions may lead.